\documentclass[conference]{IEEEtran}
\IEEEoverridecommandlockouts

\pdfoutput=1
\usepackage{hyperref}      
\usepackage{url}            
\usepackage{booktabs}       
\usepackage{amsfonts}      
\usepackage{nicefrac}      
\usepackage{microtype}      
\usepackage{amssymb}
\usepackage{graphicx}
\usepackage{enumerate}
\usepackage{amsthm}
\usepackage{amsmath,amsfonts}
\usepackage{mathtools}
\usepackage[ruled,vlined]{algorithm2e}
\usepackage{algpseudocode}
\usepackage{color}
\usepackage{accents}
\usepackage{float}
\usepackage{subfigure}
\usepackage{wrapfig,commath}
\usepackage{tikz}
\usepackage{booktabs}

\usepackage{minted}

\usepackage{cite}

\usepackage{adjustbox}
\usepackage{graphicx}

\usepackage{pifont}

\usepackage{multirow}
\usepackage[table]{colortbl}
\definecolor{lightblue}{rgb}{0.85, 0.95, 1}
\usepackage{pifont}
\usepackage{makecell}
\usepackage[table]{xcolor}

\newcommand{\bp}{\begin{proof}[\textbf{Proof}] \small}
\newcommand{\ep}{\end{proof} \normalsize}
\newcommand{\epx}{\end{proof} \small}
\newcommand{\bpa}{\begin{proofappx} \footnotesize}
\newcommand{\epa}{\end{proofappx} \small}
\newtheorem{theorem}{Theorem}
\newtheorem{assumption}{Assumption}
\newtheorem{proposition}{Proposition}
\newtheorem{corollary}{Corollary}
\newtheorem{remark}{Remark}
\newtheorem{lemma}{Lemma}

\newcommand{\be}{\begin{equation}}
\newcommand{\ee}{\end{equation}}
\newcommand{\bs}{\begin{subequations}}
\newcommand{\es}{\end{subequations}}
\newcommand{\bq}{\begin{eqnarray}}
\newcommand{\eq}{\end{eqnarray}}
\newcommand{\bqn}{\begin{eqnarray*}}
\newcommand{\eqn}{\end{eqnarray*}}

\newcommand{\ba}{\left[ \begin{array}}
\newcommand{\ea}{\\ \end{array} \right]}
\newcommand{\ben}{\begin{enumerate}}
\newcommand{\een}{\end{enumerate}}

\allowdisplaybreaks[4]
\begin{document}

\title{Collaborative In-Context Learning for Indoor \\Radio Map Estimation}

\author{
\IEEEauthorblockN{
Yuanzhe Peng\textsuperscript{1},
Jie Xu\textsuperscript{2}
}
\IEEEauthorblockA{
\textsuperscript{1}\textit{University of North Florida, Jacksonville, FL, USA}\\
\textsuperscript{2}\textit{University of Florida, Gainesville, FL, USA}
}
\thanks{
The work of Yuanzhe Peng (email: \texttt{y.peng@unf.edu}) was partially supported by UNF startup funding.
The work of Jie Xu was partially supported by the NSF under Grants 2433886, 2515982, and 2505381.
}
}

\IEEEaftertitletext{\vspace{-0.5\baselineskip}}% to decrease space between title and abstract

\maketitle

\begin{abstract}
Indoor radio map estimation (RME) is essential for wireless applications such as coverage planning, localization, and network deployment.
However, most existing methods follow a supervised learning paradigm that requires costly labeled data and task-specific model training.
Moreover, indoor environments involve privacy-sensitive information, such as building layouts and structures, that often cannot be centralized.
To address these challenges, we make the first attempt to formulate indoor RME as a collaborative in-context learning (ICL) problem.
We propose \textsc{Poly-Ripple}, a framework in which multiple clients perform ICL using large autoregressive vision models (LAVMs) without model updates.
Each client is prompted by its private demonstrations and predicts the radio map for the same scene query.
To capture the causal structure of wireless propagation outward from the transmitter (Tx), we adopt ripple-ordered autoregressive modeling.
We further introduce importance-aware aggregation to exploit heterogeneous contextual knowledge distributed across clients under the context length constraint.
We provide theoretical analysis and conduct extensive experiments on a public indoor RME benchmark.
The experimental results show that \textsc{Poly-Ripple} consistently outperforms representative ICL baselines, while aligning with our theoretical analysis.
\end{abstract}

\vspace{1mm}


\IEEEpeerreviewmaketitle

\section{Introduction}
\label{Introduction}

Accurate and cost-effective indoor radio map estimation (RME) is essential for wireless applications such as coverage planning, localization, and network deployment~\cite{zhang2024physics, lei2026importance, feng2025recent}.
A radio map provides spatially resolved estimates of path loss (PL), which quantifies signal attenuation in decibels throughout an indoor environment for a given transmitter (Tx) location.
Indoor RME is particularly challenging because of complex propagation effects such as reflection and transmission.
Traditional \textit{model-driven} approaches, including the log-distance PL model and 3GPP channel models, offer physical interpretability but often generalize poorly across diverse building layouts.

More recently, \textit{data-driven} methods have formulated RME as an image-to-image supervised learning task.
Prior studies mainly improve indoor RME from three aspects:
(1) designing encoder-decoder architectures based on convolutional neural networks (CNNs) or vision transformers (ViTs)~\cite{levie2021radiounet, bakirtzis2024rigorous, bakirtzis2022deepray, zhang2023rme, romero2024theoretical, li2025transpathnet, sun2025channel, wang2025two, viet2025spatial};
(2) customizing loss functions~\cite{lu2025sip2net, wu2025enhancing, jaiswal2025leveraging, wang2025radiodiff, cisse2025generalizable}; and
(3) applying feature fusion or data augmentation techniques~\cite{feng2025ipp, jaensch2025radio, lee2024scalable}.
Despite these advances, most existing methods still follow the supervised learning paradigm and focus on improving task-specific models.

However, practical indoor deployment exposes fundamental limitations of supervised learning.
First, obtaining labeled PL maps requires costly ray-tracing simulations or labor-intensive field measurements for each new building, making task-specific training expensive and difficult to scale.
Second, building layouts and indoor structures contain privacy-sensitive information that often cannot be centralized across sites.
This restriction makes it difficult to train a single global model using data collected from different environments.
Therefore, although supervised learning has achieved promising performance, its practical deployment for indoor RME remains limited.

\begin{figure}
\centering
\includegraphics[
    width=0.45\textwidth,
    trim=290 170 300 110,
    clip
]{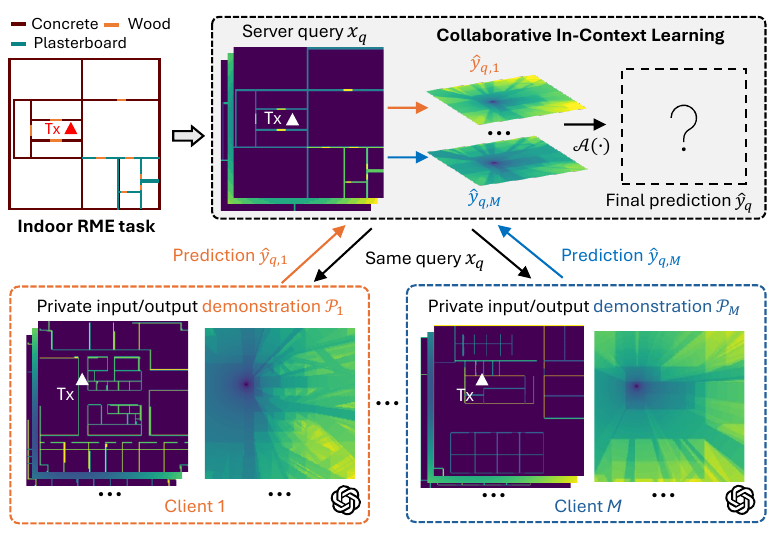}
\vspace{-4mm}
\caption{
Illustration of \textsc{Poly-Ripple}. Each client $m \in [M]$ performs ICL for the same server-issued query $\boldsymbol{x}_q$ using its private demonstration set $\mathcal{P}_{m}$. Under privacy and context length constraints, the server applies importance-aware aggregation to combine diverse client knowledge into the final prediction $\hat{\boldsymbol{y}}_q$.
}
\label{fig_case}
\vspace{-2mm}
\end{figure}

Recent advances in large language models (LLMs), such as GPTs, have demonstrated the ability to perform new tasks using only a few input-output demonstrations (or prompts), a capability known as in-context learning (ICL)~\cite{sia2024does}.
ICL eliminates the need for task-specific model updates and suggests a different approach to indoor RME.
Instead of training a model, a large autoregressive vision model (LAVM) can be prompted with a few relevant demonstrations to infer the radio map for a new scene query.
This property makes ICL particularly appealing for indoor RME, where labeled data are costly and rapid adaptation to unseen environments is desirable.
Meanwhile, indoor data are naturally distributed across buildings, floors, and organizations.
Each client maintains its own private layout and measurements, while sharing raw demonstrations may be infeasible because of privacy concerns.
Moreover, the limited context length of LAVMs restricts the number of demonstrations that can be processed within a single context.
These observations motivate the collaborative ICL setting shown in Fig.~\ref{fig_case}, where multiple clients contribute distributed contextual knowledge without sharing their private demonstrations.

\begin{figure}
\centering
  \includegraphics[width=0.47\textwidth, trim=307 200 320 165,clip]{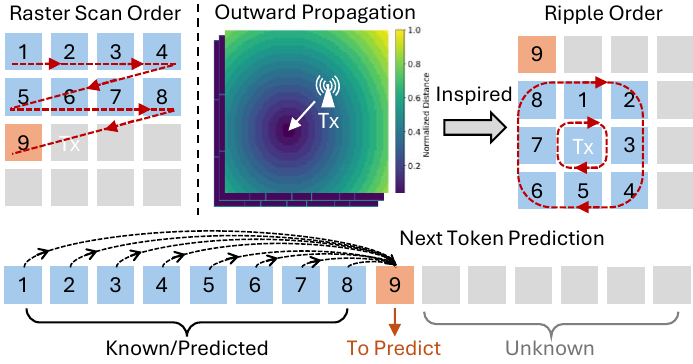}
  \vspace{-2mm}
\caption{
Raster scan order (left) is commonly used in visual autoregressive modeling \cite{xiongautoregressive}.
In contrast, inspired by the physical law of ripple-like outward signal propagation from a Tx in indoor radio map estimation, we adopt a ripple order (right) to better capture the causal structure for next-token prediction.
}
\vspace{-2mm}
  \label{fig_ripple}
\end{figure}

\textbf{Challenges.}
Despite its potential, applying collaborative ICL to indoor RME is nontrivial.
First, scene inputs and radio maps are continuous visual data and must be converted into meaningful discrete tokens before autoregressive modeling.
Second, most existing visual autoregressive models adopt raster-scan ordering, which fails to capture the causal structure of wireless propagation.
Signals propagate outward from the Tx and interact with increasingly distant obstacles, resembling ripples across still water and forming ripple-like spatial dependencies.
Third, practical deployment is distributed and subject to privacy and context length constraints.
Each client's demonstrations contain private building information and cannot be shared, while the maximum token limit allows only a few high-quality demonstrations in a single context.
The server thus requires an effective aggregation mechanism to combine heterogeneous client predictions into an accurate final estimate.

These challenges raise a \textbf{question}: \emph{how can we design a collaborative ICL framework that captures the causal structure of RME under privacy and context length constraints?}

In this paper, we propose \textsc{Poly-Ripple}, a collaborative ICL framework for indoor RME.
The term ``poly'' reflects the collaboration of multiple clients that generate ripple-structured local predictions for server-side aggregation.
On each client, \textsc{Poly-Ripple} constructs a high-quality demonstration set that considers both Tx location relevance and spatial diversity.
It then tokenizes continuous visual inputs into meaningful discrete tokens and arranges them in an outward ripple order.
As shown in Fig.~\ref{fig_ripple}, this physics-guided ordering enables the LAVM to capture the causal structure of wireless propagation for next-token prediction.
On the server side, \textsc{Poly-Ripple} applies importance-aware aggregation based on client-query similarity.
This design emphasizes predictions from more relevant demonstrations while leveraging complementary information across clients.
It exploits distributed contextual knowledge without sharing private indoor information while respecting the context length imposed by the maximum token limit.

\textbf{Contributions.}
Our contributions are summarized as follows.
\begin{itemize}
    \item We formulate indoor RME as a collaborative ICL problem beyond the conventional supervised learning paradigm.
    
    \item We propose \textsc{Poly-Ripple}, an ICL framework combining ripple autoregressive inference with importance-aware aggregation of distributed contextual knowledge.
    
    \item We provide theoretical analysis of estimation error bounds and multi-client collaboration gains.
    
    \item We conduct experiments on a public dataset, showing that \textsc{Poly-Ripple} outperforms representative ICL baselines. Experimental results align with our theoretical analysis.
\end{itemize}

\section{Related Work}
\label{sec:related_work}

\textbf{Radio Map Estimation.}
RME methods can be divided into model-driven and data-driven approaches.
Model-driven methods use analytical or physics-based models, including the log-distance PL model, 3GPP channel models, and ray-tracing simulations~\cite{feng2025recent}.
Although physically interpretable, they often generalize poorly across diverse indoor layouts.
Data-driven methods formulate RME as supervised learning from multi-channel scene inputs~\cite{levie2021radiounet, lee2024scalable, lu2025sip2net, feng2025ipp}.
They commonly adopt encoder-decoder architectures based on CNNs or ViTs~\cite{bakirtzis2024rigorous, viet2025spatial}.
Despite advances in model architectures, losses, and feature fusion, they still require costly labeled data and model training.

\textbf{In-Context Learning.}
ICL enables models to perform new tasks using a few input-output demonstrations without model updates~\cite{liu2023pre, dong2024survey, li2023practical, sun2024does, zhang2023makes, wies2023learnability, xu2024towards}.
It has been extended to visual tasks, where images are tokenized using VQ-VAE~\cite{razavi2019generating} or VQ-GAN~\cite{esser2021taming}, followed by autoregressive next-token prediction~\cite{xiongautoregressive}.
Visual ICL has shown promising performance on semantic segmentation, image inpainting, and other tasks~\cite{bai2024sequential}.
Recent studies have explored ICL for wireless applications~\cite{peng2026icc, peng2025context, tian2024visual}, but they do not address collaborative inference across distributed clients \cite{peng2024joint, peng2026distributed}, where private demonstrations cannot be shared and the maximum token limit constrains context length.

\section{Problem Formulation}
\label{sec:problem_formulation}

\textbf{Indoor RME Task.}
Consider an indoor environment discretized into a two-dimensional grid of size $H \times W$.
Let $\mathcal{N} = \{i\}_{i=1}^{N}$ denote the set of spatial coordinates, where $N = H \times W$.
Given a transmitter (Tx) location $i^{\mathrm{Tx}} \in \mathcal{N}$, indoor RME aims to predict the path loss value $\hat{y}^{i}$ at every location $i \in \mathcal{N}$, forming a radio map
$\hat{\boldsymbol{y}} = \{\hat{y}^{1},\ldots,\hat{y}^{N}\} \in \mathbb{R}^{H \times W}.$

The corresponding scene input is denoted by
$\boldsymbol{x} \in \mathbb{R}^{H \times W \times C}$,
where $C$ is the number of input channels.
Following the dataset used in this paper, each scene input contains three physically meaningful channels: the distance from each spatial location to the Tx, transmittance, and reflectance.
The transmittance and reflectance channels are determined by the material interfaces of the building layout and remain unchanged across different Tx locations within the same building layout.

\textbf{Conventional Supervised Formulation.}
Most existing RME methods adopt centralized supervised learning.
A model $f_{\theta}$ is trained to map a scene input $\boldsymbol{x}$ to the corresponding radio map,
$\hat{\boldsymbol{y}} = f_{\theta}(\boldsymbol{x}),$
where the parameters $\theta$ are optimized using a task-specific loss such as the root mean squared error.
Despite diverse designs in model architectures and loss functions, these methods generally require large labeled datasets and dedicated training for effective deployment in new environments.

\textbf{Our Collaborative ICL Formulation.}
We instead formulate indoor RME as a collaborative ICL problem with $M$ clients and a central server, enabling large-scale deployment without task-specific model training.
Each client is associated with a distinct building layout and maintains private input-output demonstrations collected from that layout.
In practice, these demonstrations cannot be centralized because they contain private information about building layouts and structures.
In addition, the LAVM has a fixed context window measured in tokens, which limits the number of demonstrations it can process.
Collaboration therefore enables distributed contextual knowledge to contribute to indoor RME without sharing private demonstrations or combining them all into one context.

Specifically, for each client $m \in \{1,\ldots,M\}$, let $\mathcal{N}_m$ denote the spatial coordinate set of its building layout.
Different Tx locations within the same layout produce different scene inputs and corresponding radio maps.
The local demonstration set of client $m$ is defined as
$\mathcal{P}_m = \{(\boldsymbol{x}_{m,k},\boldsymbol{y}_{m,k})\}_{k=1}^{K_m},$
where $(\boldsymbol{x}_{m,k},\boldsymbol{y}_{m,k})$ is an input-output demonstration associated with Tx location
$i_{m,k}^{\mathrm{Tx}} \in \mathcal{N}_m$.
These private demonstrations remain local and are not exchanged among clients or with the server.

The server, representing the user side, maintains a set of query scene inputs
$\{\boldsymbol{x}_q\}$,
where each query $\boldsymbol{x}_q$ corresponds to either a new building layout or a new Tx location.
The corresponding radio maps $\{\boldsymbol{y}_q\}$ are unavailable and must be estimated without costly task-specific model training.
The server therefore broadcasts the same query $\boldsymbol{x}_q$ to all clients.
Each client $m$ then performs ICL using an LAVM with frozen parameters $\omega$, prompted by its private demonstrations
$\mathcal{P}_m$,
\begin{align}
\scalebox{1}{$
\hat{\boldsymbol{y}}_{q,m}
=
f_{\omega}(\mathcal{P}_m,\boldsymbol{x}_q),
\quad m \in \{1,\ldots,M\}.
$}
\end{align}
Notably, the clients produce different predictions for the same query because they are prompted by different demonstrations.

The objective of collaborative ICL is to aggregate these local predictions into a final radio map $\hat{\boldsymbol{y}}_q$ under privacy and context length constraints. Specifically, the server computes
\begin{align}
\scalebox{1}{$
\hat{\boldsymbol{y}}_q
=
\mathcal{A}\big(
\hat{\boldsymbol{y}}_{q,1},
\ldots,
\hat{\boldsymbol{y}}_{q,M}
\big),
$}
\end{align}
where $\mathcal{A}(\cdot)$ denotes an aggregation operator.
For example, weighted aggregation is given by
$\hat{\boldsymbol{y}}_q = \sum_{m=1}^{M}\alpha_m\hat{\boldsymbol{y}}_{q,m},$
where $\alpha_m \geq 0$ and $\sum_{m=1}^{M}\alpha_m = 1$.
The weights reflect the relevance of each client's demonstrations to the query, such as the similarity between a representative client scene and $\boldsymbol{x}_q$.

\begin{figure*}
\centering
  \includegraphics[width=0.98\textwidth, trim=250 210 210 185, clip]{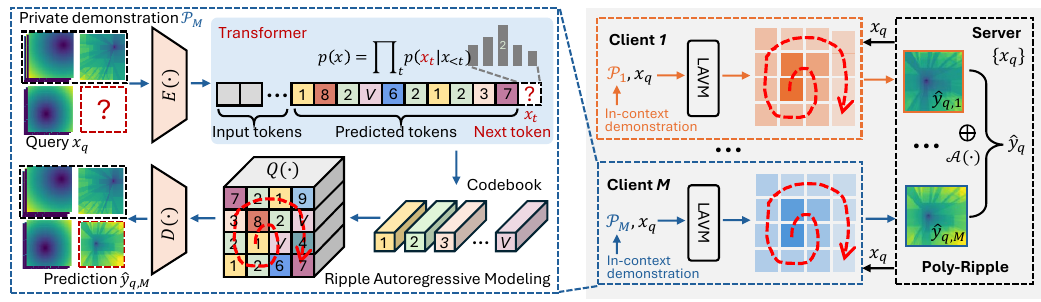}
  \vspace{-2mm}
\caption{
Each client $m \in [M]$ locally performs ICL using an LAVM, which consists of a VQ-GAN encoder and decoder for tokenization and reconstruction, and a transformer for next-token prediction.
The server aggregates the ripple-structured local predictions via $\mathcal{A}(\cdot)$ to collaboratively produce the final radio map.
}
  % \vspace{-1mm}
  \label{fig_framework}
\end{figure*}

\section{Methodology}
\label{sec:methodology}

We present the \textsc{Poly-Ripple} framework for indoor RME. 
On the client side, each client performs query dependent demonstration construction, visual tokenization, ripple ordered sequence generation, and autoregressive prediction using an LAVM. On the server side, the local predictions are aggregated using importance-aware weights, as shown in Fig.~\ref{fig_framework}.

\subsection{Client-Side In-Context Learning}
\label{sec:client_side}

Each client $m$ corresponds to a building layout and stores its demonstrations locally. Upon receiving a query $\boldsymbol{x}_q$ broadcast by the server, client $m$ performs the following four steps.

\textbf{Local Demonstration Construction.} Each client $m \in [M]$ receives the query input $\boldsymbol{x}_q$ with transmitter location $i_q^{\mathrm{Tx}}$. Let $\mathcal{D}_m=\{(\boldsymbol{x}_{m,i},\boldsymbol{y}_{m,i},i_{m,i}^{\mathrm{Tx}})\}_{i=1}^{N_m}$ denote its local dataset, where $i_{m,i}^{\mathrm{Tx}} \in \mathcal{I}_{\mathrm{Tx}}$ is a candidate transmitter position. Due to the context length constraint, client $m$ selects at most $K$ demonstrations. It first chooses the position aligned anchor $k_1=\arg\min_i\|i_{m,i}^{\mathrm{Tx}}-i_q^{\mathrm{Tx}}\|_2$, whose transmitter geometry is closest to the query. It then applies K-means clustering to $\mathcal{I}_{\mathrm{Tx}}$, forming $K$ disjoint clusters $\{\mathcal{C}_k\}_{k=1}^{K}$ with centroids $\{\mu_k\}_{k=1}^{K}$, and selects $i_{\mathrm{Tx}}^k=\arg\min_{i\in\mathcal{C}_k}\|i-\mu_k\|_2$ from each cluster. This yields $\mathcal{K}_{\mathrm{Tx}}=\{i_{\mathrm{Tx}}^1,\ldots,i_{\mathrm{Tx}}^K\}$. If the anchor is not included, the representative of the cluster containing $i_{m,k_1}^{\mathrm{Tx}}$ is replaced by the anchor. Finally, client $m$ constructs $\mathcal{P}_{q,m}=\{(\boldsymbol{x}_{m,k},\boldsymbol{y}_{m,k})\}_{k\in\mathcal{K}_{\mathrm{Tx}}}$. All operations remain local. The anchor preserves positional relevance, while the remaining representatives provide spatial diversity.

\textbf{Visual Tokenization.}
After constructing the query dependent demonstration set $\mathcal{P}_{q,m}$, client $m$ converts the demonstrations and query into discrete tokens for causal sequence modeling. We use a pretrained VQ-GAN with a fixed codebook $\mathcal{Z}$, consisting of an encoder $E(\cdot)$, a quantizer $Q(\cdot)$, and a decoder $D(\cdot)$. For any input $\boldsymbol{x}$, the encoder produces a latent tensor $\hat{\boldsymbol{z}} = E(\boldsymbol{x}) \in \mathbb{R}^{h \times w \times b}$, where $h = H / f$, $w = W / f$, $f$ is the downsampling factor, and $b$ is the embedding dimension. Each latent vector $\hat{\boldsymbol{z}}^{i,j}$ is mapped to the index of its nearest codeword in $\mathcal{Z} = \{ \boldsymbol{z}_{v} \}_{v=1}^{V} \subset \mathbb{R}^{b}$ according to
$Q(\hat{\boldsymbol{z}}^{i,j})
=
\arg\min_{v \in [V]}
\bigl\|
\hat{\boldsymbol{z}}^{i,j}
-
\boldsymbol{z}_{v}
\bigr\|.$
This produces a discrete token grid $\boldsymbol{z} \in \{1,\ldots,V\}^{h \times w}$. Client $m$ applies the tokenizer to all inputs and outputs in $\mathcal{P}_{q,m}$ and to the query input $\boldsymbol{x}_{q}$, yielding their corresponding token grids.

\textbf{Ripple Ordering.}
To reflect the causal structure of wireless propagation, tokens are organized in a ripple order rather than a raster scan order. Specifically, tokens are arranged by increasing radial distance from the transmitter, so autoregressive prediction proceeds from nearby regions to farther regions.

For a demonstration pair $(\boldsymbol{z}^{\mathrm{in}}_{k}, \boldsymbol{z}^{\mathrm{out}}_{k})$ with transmitter location $i^{\mathrm{Tx}}_{k}$ mapped to the latent token grid, the grid positions are represented as $\tilde{\mathcal{I}}_{k} = \{ \tilde{i}_{k,t} \}_{t=1}^{T}$, where $T = h \cdot w$. The radial distance of each position is
$r_{k,t}
=
\bigl\|
\tilde{i}_{k,t}
-
i_{k}^{\mathrm{Tx}}
\bigr\|_{2}.$
Given increasing radial thresholds $0=r_0<r_1<\cdots<r_{L_{\max}}$, the $\ell$th ripple layer is defined as
\begin{align}
\scalebox{0.95}{$
\mathcal{R}^{\ell}_{k}
=
\bigl\{
\tilde{i}_{k,t} \in \tilde{\mathcal{I}}_{k}
\;\big|\;
r_{\ell-1}
<
\|
\tilde{i}_{k,t}
-
i_{k}^{\mathrm{Tx}}
\|_{2}
\le
r_{\ell}
\bigr\},
\ell \in [L_{\max}].
$}
\end{align}
Let $R_k^\ell=|\mathcal{R}_k^\ell|$. Within each ripple layer, input and output tokens at the same position are interleaved to associate environmental context with local propagation behavior. For $\tilde{i}_{\ell,j} \in \mathcal{R}^{\ell}_{k}$, the interleaved token pair is
$\boldsymbol{s}^{\ell,j}_{k}
=
[
\boldsymbol{z}^{\mathrm{in}}_{k}(\tilde{i}_{\ell,j}),
\boldsymbol{z}^{\mathrm{out}}_{k}(\tilde{i}_{\ell,j})
].$
The sequence for layer $\ell$ is
\begin{align}
\boldsymbol{s}^{\ell}_{k}
=
[
\boldsymbol{z}^{\mathrm{in}}_{k}(\tilde{i}_{\ell,1}),
\boldsymbol{z}^{\mathrm{out}}_{k}(\tilde{i}_{\ell,1}),
\ldots,
\boldsymbol{z}^{\mathrm{in}}_{k}(\tilde{i}_{\ell,R^{\ell}_{k}}),
\boldsymbol{z}^{\mathrm{out}}_{k}(\tilde{i}_{\ell,R^{\ell}_{k}})
].
\label{eq:interleaved_pairs}
\end{align}
We concatenate the layers in ascending radial order and append a separator token after each layer, including empty layers, to maintain a consistent sequence structure. The sequence for demonstration $k$ is
$\boldsymbol{s}_{k}
=
[
\boldsymbol{s}^{1}_{k},\mathrm{SEP},
\boldsymbol{s}^{2}_{k},\mathrm{SEP},
\ldots,
\boldsymbol{s}^{L_{\max}}_{k},\mathrm{SEP}
].$
Finally, client $m$ concatenates all selected demonstration sequences to construct
$\mathcal{P}^{*}_{q,m}
=
[
\boldsymbol{s}_{k_{1}},
\boldsymbol{s}_{k_{2}},
\ldots,
\boldsymbol{s}_{k_{K_{m}}}
].$

\textbf{Autoregressive Prediction.}
Given the ripple ordered prompt context $\mathcal{P}^{*}_{q,m}$ and query input tokens $\boldsymbol{z}^{\mathrm{in}}_{q}$, client $m$ performs autoregressive prediction using a causal transformer with fixed parameters. The transformer sequentially generates the client specific output tokens $\boldsymbol{z}^{\mathrm{out}}_{q,m}$, conditioned on the prompt context, the query input, and previously generated output tokens. Formally, the joint likelihood factorizes as
\begin{align}
\scalebox{0.9}{$
p\bigl(
\boldsymbol{z}^{\mathrm{out}}_{q,m}
\mid
\mathcal{P}^{*}_{q,m},
\boldsymbol{z}^{\mathrm{in}}_{q}
\bigr)
=
\prod_{t=1}^{T}
p(
z^{\mathrm{out}}_{q,m,t}
|
\boldsymbol{z}^{\mathrm{out}}_{q,m,1:(t-1)},
\mathcal{P}^{*}_{q,m},
\boldsymbol{z}^{\mathrm{in}}_{q}
), 
$}
\end{align}
where $T=h\cdot w$ is the number of tokens in the latent grid, and $z^{\mathrm{out}}_{q,m,t}$ is generated at the $t$th position in the ripple ordered sequence. The causal attention mask prevents access to future output tokens, so the generated outputs follow the outward ripple order while conditioning on the observed query input and prompt context. The generated output $\boldsymbol{z}^{\mathrm{out}}_{q,m}$ is rearranged into a discrete token grid in $\{1,\ldots,V\}^{h\times w}$. A lookup operator $Z(\cdot)$ maps each token index to its codeword in $\mathcal{Z}=\{\boldsymbol{z}_v\}_{v=1}^{V}$, and the VQ-GAN decoder reconstructs the local radio map as
\begin{align}
\scalebox{1}{$
\hat{\boldsymbol{y}}_{q,m}
=
D \bigl( Z(\boldsymbol{z}^{\mathrm{out}}_{q,m}) \bigr), \quad \forall\, m\in[M].
$}
\end{align}
Finally, only $\hat{\boldsymbol{y}}_{q,m}$ is transmitted, while the private demonstrations and intermediate token sequences remain stored locally.

\subsection{Server-Side Importance-Aware Aggregation}
\label{sec:server_side}

After all clients generate local predictions using their private prompts, the server aggregates them to obtain the final estimate. Uniform averaging is straightforward, but it may be suboptimal because the local prompt sets can differ substantially in their relevance to the query. We thus express the aggregation rule as
\begin{align}
\scalebox{1}{$
\hat{\boldsymbol{y}}_{q}
=
\mathcal{A}\big(
\hat{\boldsymbol{y}}_{q,1}, \ldots, \hat{\boldsymbol{y}}_{q,M}
\big),
$}
\end{align}
where $\mathcal{A}(\cdot)$ is a weighted aggregator that assigns each client an importance weight based on its relevance to the query. Let $\boldsymbol{x}_{m}$ denote a representative scene input of client $m$, such as the anchor selected for the current query. Given the query scene $\boldsymbol{x}_{q}$, their cosine similarity is
$S_{m}
=
\frac{
\langle
\phi(\boldsymbol{x}_{m}),
\phi(\boldsymbol{x}_{q})
\rangle
}{
\|
\phi(\boldsymbol{x}_{m})
\|_{2}
\,
\|
\phi(\boldsymbol{x}_{q})
\|_{2}
},$
where $\phi(\cdot)$ is a shared scene encoder. This score measures the structural relevance of client $m$ to the query. The embedding $\phi(\boldsymbol{x}_{m})$ or the resulting score $S_m$ is computed locally and transmitted to the server, while the original scene remains private.
The similarity scores are converted into aggregation weights through a temperature scaled softmax function,
\begin{align}
% \scalebox{1}{$
\alpha_{m}
=
\frac{
\exp(S_{m}/\tau)
}{
\sum_{m'=1}^{M} \exp(S_{m'}/\tau)
}, \quad \forall\, m\in[M],
% $}
\end{align}
where $\tau>0$ controls the concentration of the weights. A smaller $\tau$ emphasizes the most relevant clients, whereas a larger $\tau$ distributes the weights more evenly.

The final radio map prediction for server query $\boldsymbol{x}_{q}$ is
\begin{align}
\scalebox{1}{$
\hat{\boldsymbol{y}}_{q}
=
\sum\nolimits_{m=1}^{M}
\alpha_{m}
\hat{\boldsymbol{y}}_{q,m}.
$}
\end{align}

This importance-aware aggregation emphasizes predictions from relevant clients while reducing the influence of less informative ones. Other relevance measures, including confidence based or hybrid criteria, can also be incorporated without changing the overall framework.

\section{Theoretical Analysis}
\label{sec:theoretical}

\subsection{Estimation Error Analysis}
\label{subsec:estimation-error}

We introduce a surrogate model to characterize the implicit adaptation induced by contextual conditioning in ICL. 
We consider a common server-issued query broadcast to all clients, whose private local demonstrations may produce different predictions. Each client $m \in [M]$ observes $K$ independent and identically distributed demonstrations $\{(\boldsymbol{x}_{k}, \boldsymbol{y}^{m}_{k})\}_{k=1}^K$. For any subset $\mathcal{S} \subseteq [d] \triangleq \{1,\ldots,d\}$, let $\boldsymbol{x}_{\mathcal{S}} \in \mathbb{R}^{|\mathcal{S}|}$ denote the restriction of the input vector $\boldsymbol{x}$ to the coordinates indexed by $\mathcal{S}$. We begin by making the following mild assumptions.

\begin{assumption}
% [Shared linear conditional mean across clients]
\label{ass:linear-multi}
For every output coordinate $i \in [d]$ and every subset $\mathcal{S} \subseteq [d]$,
there exist coefficients
$w^{\star}_{i,\mathcal{S}} \in \mathbb{R}^{|\mathcal{S}|}$
and client-specific fluctuation terms
$\epsilon^{m}_{i}$
such that, for all clients $m \in [M]$,
\begin{align}
\label{eq:linear-multi-clean}
\scalebox{0.95}{$
y^{m}_{i}
=
\boldsymbol{x}_{\mathcal{S}}^{\top} w^{\star}_{i,\mathcal{S}}
+
\epsilon^{m}_{i},
\
\mathbb{E}\!\left[\epsilon^{m}_{i} | \boldsymbol{x}_{\mathcal{S}}\right] = 0,
\
\mathrm{Var}\!\left(\epsilon^{m}_{i} | \boldsymbol{x}_{\mathcal{S}}\right) = \sigma_i^2.
$}
\end{align}
\end{assumption}

\begin{assumption}
\label{ass:isotropic-multi}
The inputs are standardized with zero mean and identity covariance:
$\mathbb{E}[\boldsymbol{x}] = 0,$
$\mathrm{Cov}(\boldsymbol{x}) = I_d .$
\end{assumption}

\begin{assumption}
\label{ass:monotone-sigma-multi}
There exists a physical ordering of output coordinates by increasing distance to the transmitter.
Without loss of generality, indices are ordered so that $i=1$ corresponds to the closest location and $i=d$ to the farthest.
The dynamic range levels satisfy
$\sigma_1^2 \ge \sigma_2^2 \ge \cdots \ge \sigma_d^2.$
\end{assumption}

\begin{assumption}
\label{ass:delta-const-multi}
For any client $m \in [M]$ and any two coordinate subsets $\mathcal{S} \neq \mathcal{S}'$, the initialization error satisfies
\begin{align}
% \scalebox{1}{$
\|\delta_{\mathcal{S},m}\|_2^2
=
\|\delta_{\mathcal{S}',m}\|_2^2
=
\|\delta_m\|_2^2.   
% $}
\end{align}
\end{assumption}

Assumption~\ref{ass:linear-multi} models the shared propagation relationship through $w^{\star}_{i,\mathcal{S}}$, while $\epsilon_i^m$ captures variations induced by different local contexts. The coordinate-specific variance $\sigma_i^2$ represents the signal dynamic range. Assumption~\ref{ass:isotropic-multi} standardizes the input covariance, while the proof of Lemma~\ref{lemma:onestep-error-multi} additionally uses the isotropic fourth-moment identity specified in the proof section. Assumption~\ref{ass:monotone-sigma-multi} reflects the decreasing signal variability from the Tx outward and motivates ripple-ordered estimation. Finally, Assumption~\ref{ass:delta-const-multi} characterizes each client's initialization mismatch by $\|\delta_m\|_2^2$, thereby isolating the effect of the estimation order.

\begin{lemma}[One-step estimation error on client $m$]
\label{lemma:onestep-error-multi}
Under Assumptions~\ref{ass:linear-multi} and~\ref{ass:isotropic-multi},
for any client $m \in [M]$, any output coordinate $i \in [d]$, and any subset $\mathcal{S} \subseteq [d]$,
after one (implicit) update step with step size $\eta$ using all $K$ demonstrations on client $m$,
the expected squared estimation error satisfies
\begin{align}
\mathbb{E}\Bigl[
\bigl(
y_i^{m}
-
\boldsymbol{x}_{\mathcal{S}}^{\top} w^{1}_{i,\mathcal{S}}
\bigr)^2
\Bigr]
&\le
\Bigl(1 + \eta^{2}\tfrac{|\mathcal{S}|}{K}\Bigr)\sigma_i^2
+ (1-\eta)^2 \|\delta_{\mathcal{S},m}\|_2^2
\nonumber\\
&\quad
+ \eta^{2}\tfrac{|\mathcal{S}|(|\mathcal{S}|+1)}{K}\,
\|\delta_{\mathcal{S},m}\|_2^2 .
\label{eq:lemma1-multi}
\end{align}
\end{lemma}

\subsection{Optimal Estimation Order Analysis}
\label{subsec:optimal-order}

We next examine how the estimation order affects the cumulative error on each client. Let $o$ be a permutation of $\{1,\ldots,d\}$ specifying the output order. At step $t$, coordinate $o(t)$ is estimated using the cumulative feature set
$\mathcal{S}(t) \triangleq \{o(1),\ldots,o(t)\},
|\mathcal{S}(t)| = t.$
Applying Lemma~\ref{lemma:onestep-error-multi} at each step and summing over $t \in [d]$ gives the following bound.

\begin{theorem}[Total estimation error bound on a single client $m$]
\label{thm:total-error-multi}
Under Assumptions~\ref{ass:linear-multi}--\ref{ass:delta-const-multi},
for any client $m \in [M]$ and any estimation order $o$,
the total expected prediction error accumulated over all $d$ sequential estimation steps satisfies
\begin{align}
&\sum\nolimits_{t=1}^d 
\mathbb{E}\Bigl[
\bigl(
y^{m}_{o(t)} - \boldsymbol{x}_{\mathcal{S}(t)}^{\top} w^{1}_{\mathcal{S}(t)}
\bigr)^2
\Bigr]
\nonumber\\
&\le
\sum\nolimits_{i=1}^d \sigma_i^2
+ (1-\eta)^2 \|\delta_m\|_2^2\, d
\nonumber\\
&\quad
+ \eta^2 \frac{\|\delta_m\|_2^2}{K}\, \frac{d(d+1)(d+2)}{3}
+ \frac{\eta^2}{K} \sum\nolimits_{t=1}^d t\,\sigma_{o(t)}^2 .
\label{eq:thm-total-multi}
\end{align}
\end{theorem}

\begin{remark}
\label{rem:order-dependent}
\textnormal{
In~\eqref{eq:thm-total-multi}, the first three terms are independent of $o$, since $o$ is a permutation and $\sum_{t=1}^{d}t(t+1)=\frac{d(d+1)(d+2)}{3}$. Therefore, the estimation order affects the bound only through
$T(o) \triangleq \frac{\eta^{2}}{K}\sum_{t=1}^{d} t\,\sigma_{o(t)}^2.$
Because later coordinates receive larger coefficients $t$, minimizing the bound requires larger-variance coordinates to be estimated earlier.
}
\end{remark}

\begin{proposition}[Optimal estimation order]
\label{prop:optimal-order-multi}
Under Assumption~\ref{ass:monotone-sigma-multi}, for any permutation $o$ of $\{1,\ldots,d\}$,
the order-dependent term
$T(o) \triangleq \sum_{t=1}^{d} t\,\sigma_{o(t)}^2$
satisfies
$\sum_{t=1}^{d} t\,\sigma_{o(t)}^2
\ge
\sum_{t=1}^{d} t\,\sigma_{(d+1-t)}^2,$
with equality if and only if the permutation $o$ orders the sequence
$\{\sigma_i^2\}_{i=1}^{d}$ in nonincreasing order.
\end{proposition}

\begin{remark}
\textnormal{
Proposition~\ref{prop:optimal-order-multi} shows that the order-dependent error is minimized by estimating coordinates with larger signal variability earlier. This ordering is common to all clients because $\{\sigma_i^2\}$ is shared, while client heterogeneity appears only through $\delta_m$ in the order-independent terms. Under
$\sigma_1^2 \ge \sigma_2^2 \ge \cdots \ge \sigma_d^2$,
the optimal permutation is $o^\star(t)=t$, which corresponds to the proposed outward ripple order.
}
\end{remark}

\begin{corollary}[Closed-form estimation error on a single client $m$]
\label{corollary:ripple-multi}
Under Assumptions~\ref{ass:linear-multi}--\ref{ass:delta-const-multi},
by choosing the optimal estimation order in Proposition~\ref{prop:optimal-order-multi},
the total expected prediction error accumulated over all $d$ sequential estimation steps satisfies
\begin{align}
\sum\nolimits_{t=1}^d &
\mathbb{E} \Bigl[
(
y^{m}_{t} - \boldsymbol{x}_{\mathcal{S}(t)}^{\top} w^{1}_{\mathcal{S}(t)}
)^2
\Bigr] \le
\sum\nolimits_{i=1}^d (1+\frac{\eta^2 i}{K})\sigma_i^2
\nonumber\\
&\quad
+ d\Bigl((1-\eta)^2
+ \frac{\eta^2}{K}\,\frac{(d+1)(d+2)}{3}\Bigr)\|\delta_m\|_2^2 .
\label{eq:corollary-ripple-multi}
\end{align}
\end{corollary}

\begin{remark}
\textnormal{
Corollary~\ref{corollary:ripple-multi} specializes Theorem~\ref{thm:total-error-multi} to ripple ordering, for which the order-dependent term becomes
$T(o)=\tfrac{\eta^2}{K}\sum_{i=1}^d i\,\sigma_i^2$.
The bound holds for each client, with heterogeneity entering through $\|\delta_m\|_2^2$. Physically, estimating high-variability near-Tx regions first provides contextual information for subsequent prediction of more stable far-field regions.
}
\end{remark}

\subsection{Aggregation Analysis}
\label{subsec:aggregation}

We next analyze the estimation error obtained by aggregating the predictions of $M$ clients.

\begin{assumption}
\label{ass:cross-client-uncorr}
For any step $t \in [d]$ and any two distinct clients $m \neq m'$,
the client prediction errors w.r.t. the ground truth are uncorrelated:
$\mathbb{E}[
(
y^{m}_{o(t)} - \boldsymbol{x}_{\mathcal{S}(t)}^\top w_{\mathcal{S}(t)}
)
(
y^{m'}_{o(t)} - \boldsymbol{x}_{\mathcal{S}(t)}^\top w_{\mathcal{S}(t)}
)
]
= 0 .$
\end{assumption}

Assumption~\ref{ass:cross-client-uncorr} models clients whose prediction errors contain no shared systematic component. This condition is reasonable when clients use disjoint local demonstrations and their remaining prediction fluctuations are independent across clients. 
% It eliminates cross-client terms in the aggregated squared error.

\begin{theorem}[Closed-form aggregated estimation error via uniform aggregation]
\label{thm:agg-total-error}
Under Assumptions~\ref{ass:linear-multi}--\ref{ass:cross-client-uncorr},
let each client adopt the optimal ripple estimation order characterized in
Proposition~\ref{prop:optimal-order-multi}.
Then, under uniform averaging across $M$ clients, the total prediction error satisfies
\begin{align}
& \sum\nolimits_{t=1}^{d}
\mathbb{E}\Big[
\Big(
\frac{1}{M}\sum\nolimits_{m=1}^{M}
y^{m}_{t}
-
\boldsymbol{x}_{\mathcal{S}(t)}^\top w_{\mathcal{S}(t)}
\Big)^2
\Big]
\nonumber\\
&\le 
\frac{1}{M^2}\sum\nolimits_{m=1}^{M}
\Big[
\sum\nolimits_{i=1}^{d}
\Bigl(1+\frac{\eta^2 i}{K}\Bigr)\sigma_i^2
\nonumber \\
& \qquad +
d\Bigl((1-\eta)^2
+
\frac{\eta^2}{K}\,\frac{(d+1)(d+2)}{3}\Bigr)
\|\delta_m\|_2^2
\Big].
\label{eq:agg-total-one-eq}
\end{align}
\end{theorem}

\begin{remark}
\textnormal{
Theorem~\ref{thm:agg-total-error} shows that uniform aggregation preserves the ripple-order error structure while reducing its magnitude through $\frac{1}{M^2}\sum_{m=1}^{M}(\cdot)$. When clients have comparable prediction errors, this yields an effective $1/M$ gain from averaging uncorrelated errors. 
}
\end{remark}

\begin{theorem}[Closed-form aggregated estimation error via weighted aggregation]
\label{thm:total-error-agg}
Under Assumptions~\ref{ass:linear-multi}--\ref{ass:cross-client-uncorr},
let each client adopt the optimal ripple estimation order characterized in
Proposition~\ref{prop:optimal-order-multi}.
For any nonnegative aggregation weights
$\{\alpha_m\}_{m=1}^{M}$ satisfying $\sum_{m=1}^{M}\alpha_m = 1$,
define the aggregated prediction 
% at coordinate $t$ 
as
$y^{\mathcal{A}}_{t}
\triangleq
\sum_{m=1}^{M}\alpha_m\, y^{m}_{t}.$
Then the total prediction error after weighted aggregation satisfies
\begin{align}
& \sum\nolimits_{t=1}^{d}
\mathbb{E}\Big[
\big(
y^{\mathcal{A}}_{t}
-
\boldsymbol{x}_{\mathcal{S}(t)}^\top w_{\mathcal{S}(t)}
\big)^2
\Big]
\nonumber \\
& \le
\sum\nolimits_{m=1}^{M}\alpha_m^2
\Big[
\sum\nolimits_{i=1}^{d}
\Bigl(1+\frac{\eta^2 i}{K}\Bigr)\sigma_i^2
\nonumber \\
& \qquad +
d\Bigl((1-\eta)^2
+
\frac{\eta^2}{K}\,\frac{(d+1)(d+2)}{3}\Bigr)
\|\delta_m\|_2^2
\Big].
\label{eq:agg-total-wt-closed}
\end{align}
\end{theorem}

\begin{remark}
\textnormal{
Theorem~\ref{thm:total-error-agg} extends uniform averaging to weighted aggregation. The bound weights each client's ripple-order error by $\alpha_m^2$, showing that aggregation performance depends on both client quality and the weight distribution. The shared term $\sum_{i=1}^{d}(1+\eta^2 i/K)\sigma_i^2$ captures coordinate-specific signal variability, while the term involving $\|\delta_m\|_2^2$ captures client-specific mismatch. Therefore, assigning larger weights to clients with smaller mismatch can reduce the aggregated error bound.
}
\end{remark}

Theorem~\ref{thm:total-error-agg} motivates nonuniform \textbf{importance-aware aggregation} in \textsc{Poly-Ripple}. Ideally, clients whose local demonstrations better match the server-issued query should receive larger weights, as they are expected to have smaller mismatch-induced errors. We next translate the bound into an importance-aware design principle.
Specifically, the right-hand side of \eqref{eq:agg-total-wt-closed} is quadratic in the aggregation weights. We define the client-dependent estimation error proxy
\begin{align}
\scalebox{0.88}{$
r_m^2
\triangleq
\sum\nolimits_{i=1}^{d}
(1+\frac{\eta^2 i}{K})\sigma_i^2
+
d((1-\eta)^2
+
\frac{\eta^2}{K}\,\frac{(d+1)(d+2)}{3})\|\delta_m\|_2^2,
% \qquad m\in[M],
\label{eq:rm-def-after-thm3-v2}
$}
\end{align}
and consider the following optimization problem:
\begin{align}
\mathbf{P1}\qquad
\min_{\{\alpha_m\}_{m=1}^{M}}\ \ &
\sum\nolimits_{m=1}^{M}\alpha_m^2 r_m^2
\label{eq:P1-alpha-opt-v2}
\\
\textnormal{s.t.}\ \ &
\alpha_m\ge 0,\quad \forall\, m\in[M],
\nonumber \\
& \sum\nolimits_{m=1}^{M}\alpha_m=1.
\nonumber
\end{align}

Problem \textbf{P1} is important for two reasons. First, it provides
an explicit and principled objective that is implied by
our theoretical bound, and therefore clarifies what a better
aggregation rule should optimize. Second, its closed-form
solution reveals an inverse-risk allocation rule, which directly
inspires our importance-aware weighting strategy in \textsc{Poly-Ripple}. 

\begin{proposition}
\label{prop:alpha-optimal-after-thm3-v3}
Assume $r_m^2>0$ for all $m\in[M]$.  
The solution to the optimization problem in
\eqref{eq:P1-alpha-opt-v2} follows an inverse-risk allocation rule and can be
expressed as
\begin{align}
% \scalebox{1}{$
\alpha_m^\star
\ \propto\
\frac{1}{r_m^2}\cdot \kappa,
\quad \forall\, m\in[M],
% $}
\label{eq:alpha-propto-with-kappa-v3}
\end{align}
where the (client-independent) 
% normalization
factor $\kappa$ is defined by
$\kappa
\triangleq
(\sum\nolimits_{m'=1}^{M}\frac{1}{r_{m'}^2})^{-1},$
so that $\sum_{m=1}^{M}\alpha_m^\star=1$.
\end{proposition}

\begin{remark}
\textnormal{
Proposition~\ref{prop:alpha-optimal-after-thm3-v3} provides a weight design principle implied by Theorem~\ref{thm:total-error-agg}.
The optimal aggregation weight scales inversely with the client-specific proxy
$r_m^2$, so clients with smaller $r_m^2$ receive larger weights.
Under the same estimation order and shared values of $K$, $\eta$, and
$\{\sigma_i^2\}_{i=1}^{d}$, $r_m^2$ differs across clients only through the
mismatch term $\|\delta_m\|_2^2$ in its definition, as shown in
\eqref{eq:rm-def-after-thm3-v2} and the bracketed term in
\eqref{eq:agg-total-wt-closed}.
Specifically, the remaining term
$\sum_{i=1}^{d}(1+\eta^2 i/K)\sigma_i^2$ is shared across clients and therefore
does not affect the relative ordering among $\{r_m^2\}$.
Consequently, the inverse-risk rule indicates that clients with smaller mismatch
magnitudes $\|\delta_m\|_2^2$, and hence smaller $r_m^2$, should be assigned
larger weights.
In practice, the server cannot directly observe $\|\delta_m\|_2^2$ because the
client prompts and intermediate token sequences remain private.
Nevertheless, the proposition specifies the desired monotonic structure for
weight assignment: $\alpha_m$ should decrease with a client-specific risk
indicator.
This motivates our method to adopt a computable relevance score, such as the
cosine similarity between a representative client scene and the query scene,
as a practical proxy for client risk.
A client with higher scene relevance is expected to induce a smaller mismatch
with the server-issued query and is therefore assigned a larger aggregation
weight, whereas less relevant clients are downweighted.
}
\end{remark}

\section{Experiments}
\label{sec:Experiments}

We conduct extensive experiments to answer the following three questions.
The corresponding results and analyses are presented in
Sections~\ref{Results and Analysis}, \ref{Ablation Studies}, and
\ref{subsec:theory-practice}, respectively.

\textbf{Q1:} How does \textsc{Poly-Ripple} compare with ICL baselines in terms of quantitative and qualitative performance?

\textbf{Q2:} How do the key parameters affect its performance?

\textbf{Q3:} Do the empirical results support the theoretical analysis?

\subsection{Experimental Setup}

\textbf{Dataset.}
We adopt a publicly available indoor radio map dataset~\cite{Dataport}.
Each sample is associated with a building layout, a frequency band, an antenna radiation pattern, and a Tx location.
The dataset contains 25 building layouts, three frequency bands, and five antenna radiation patterns.
For each combination of building layout, frequency band, and antenna radiation pattern, radio maps are provided for 50 to 80 Tx locations, resulting in a total of 27,750 samples.
Each input consists of three channels representing reflectance (dB), transmittance (dB), and distance to the Tx (m), while the output is the corresponding radio map.

\textbf{Baselines.}
\textit{(1) Local ICL.}
Each client independently predicts the query issued by the server using its private demonstrations and ripple-ordered autoregressive modeling.
The resulting local predictions are evaluated without server-side aggregation.
\textit{(2) Collaborative ICL with Uniform Aggregation.}
Multiple clients independently predict the same query using their private demonstrations.
The server then uniformly averages their local predictions by assigning each client a weight of $1/M$.

\textbf{Implementation and Metrics.}
We implement our framework using \texttt{PyTorch 2.6} and conduct all experiments on NVIDIA B200 GPUs.
We deploy $20$ clients, denoted by C1 to C20, and assign each client a unique building layout.
The remaining five building layouts, denoted by C21 to C25, are maintained by the server as unseen target layouts representing the user side.
We use a VQ-GAN~\cite{hao2024data} for tokenization and a LLaMA model~\cite{touvron2023llama} for autoregressive prediction.
With a downsampling factor of $f=16$, the VQ-GAN tokenizes each resized $256 \times 256$ image into $256$ discrete tokens using a codebook containing $8{,}192$ entries, each with dimension $64$.
For autoregressive modeling, we adopt the LLaMA-7B checkpoint released with LVM~\cite{bai2024sequential}, which was pretrained on $50$ diverse datasets with $420$ billion tokens and demonstrates strong generalization to unseen tasks.
We use root mean squared error (RMSE) as the evaluation metric, where $\downarrow$ indicates that a lower value is better.
We report mean RMSE and standard deviation over five independent runs.

\subsection{Results and Analysis}
\label{Results and Analysis}

To answer \textbf{Q1}, we evaluate \textsc{Poly-Ripple} with $M=20$ clients.
Each client maintains private demonstrations from one building layout and constructs a local prompt containing $K=3$ paired input-output demonstrations.
Queries are sampled from five unseen layouts, and broadcast to all clients for local ICL.

\vspace{-3mm}
\begin{figure}[H]
\centering
\includegraphics[
    width=0.45\textwidth,
    trim=180 60 175 65,
    clip
]{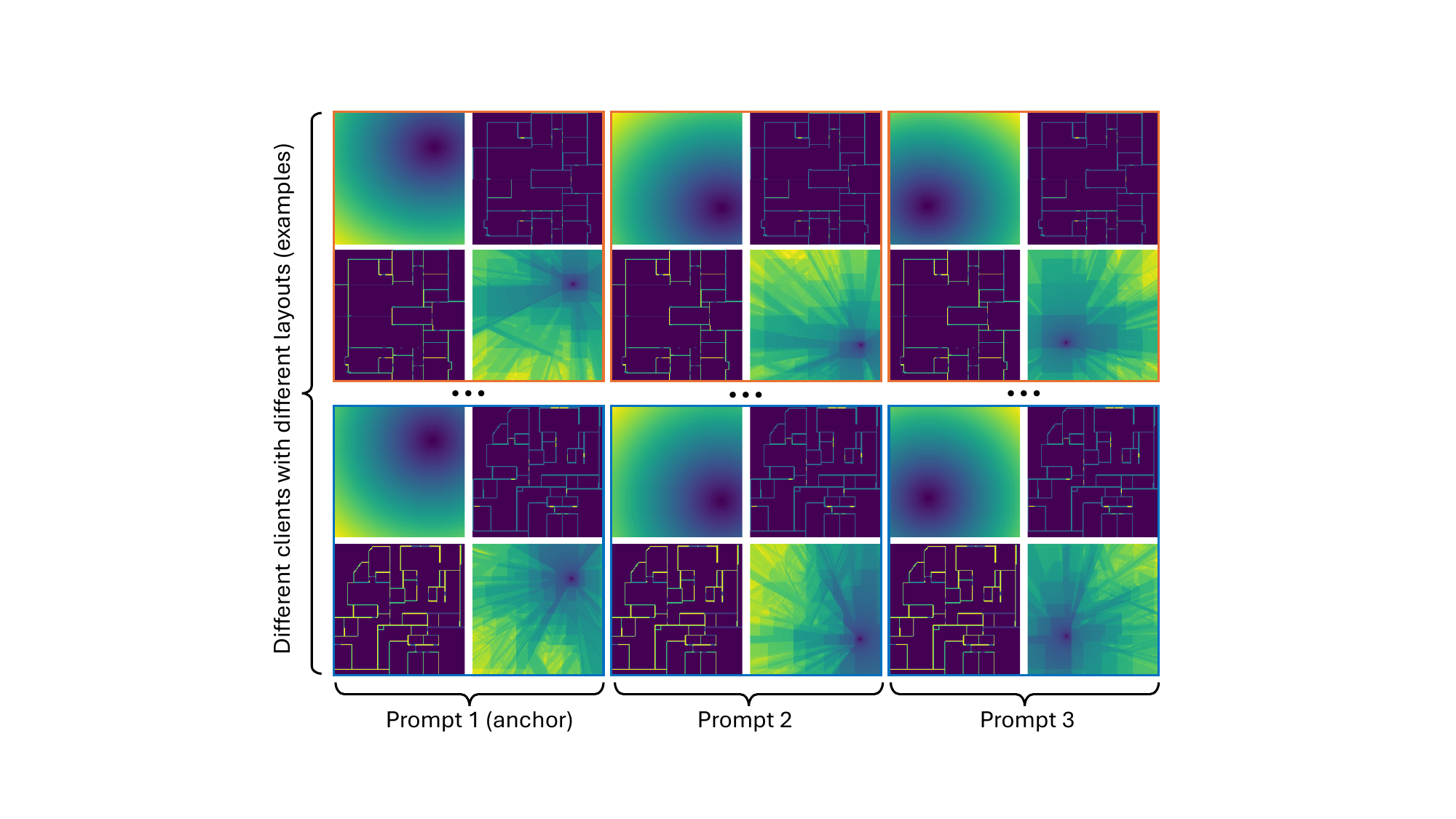}
\vspace{-2mm}
\caption{Local prompt construction with $K=3$ on different clients.}
\vspace{-2mm}
\label{Prompts_selection}
\end{figure}

\textbf{Diversity in local prompt construction.}
For each query, client $m$ constructs its local prompt using the two-step strategy described in Section~\ref{sec:client_side}.
The first demonstration is selected as an anchor whose Tx location is closest to the query Tx location in normalized coordinates.
The remaining $K-1$ demonstrations are selected to provide spatially diverse Tx locations within the same client layout.
As shown in Fig.~\ref{Prompts_selection}, the first column contains the anchor demonstrations, while the remaining columns show the diverse demonstrations selected by different clients.
Fig.~\ref{BTx} further visualizes the selected Tx locations for a representative client.
The anchor provides a propagation context that is geometrically aligned with the query, while the remaining demonstrations introduce complementary propagation patterns from other Tx locations.
This design balances query relevance and contextual diversity under the limited context length of the LAVM.

% \vspace{-3mm}
\begin{figure}
\centering
\includegraphics[
    width=0.49\textwidth,
    trim=260 300 265 152,
    clip
]{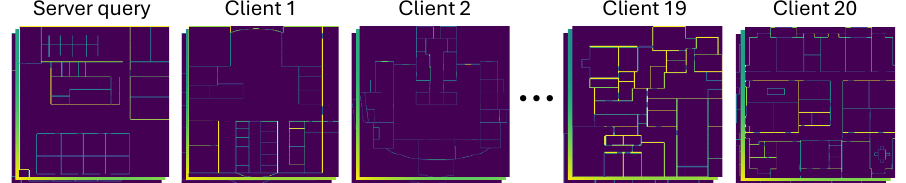}
\vspace{-5mm}
\caption{A query layout and client layouts with varying structural relevance.}
\vspace{-3mm}
\label{all_score}
\end{figure}

\textbf{Heterogeneous relevance across clients.}
Although all clients receive the same query, their private demonstrations are collected from different building layouts.
As illustrated in Fig.~\ref{all_score}, some client layouts are structurally more similar to the query layout than others.
Consequently, their demonstrations provide different levels of contextual relevance and lead to heterogeneous local prediction quality.
A client with demonstrations that better match the query environment is more likely to provide useful information about its propagation patterns.
In contrast, uniform aggregation ignores this heterogeneity and assigns equal importance to all local predictions.
This observation is consistent with the client-dependent mismatch characterized in our theoretical analysis.

\vspace{-2mm}
\begin{figure}[H]
\centering
\includegraphics[
    width=0.49\textwidth,
    trim=5 10 1 1,
    clip
]{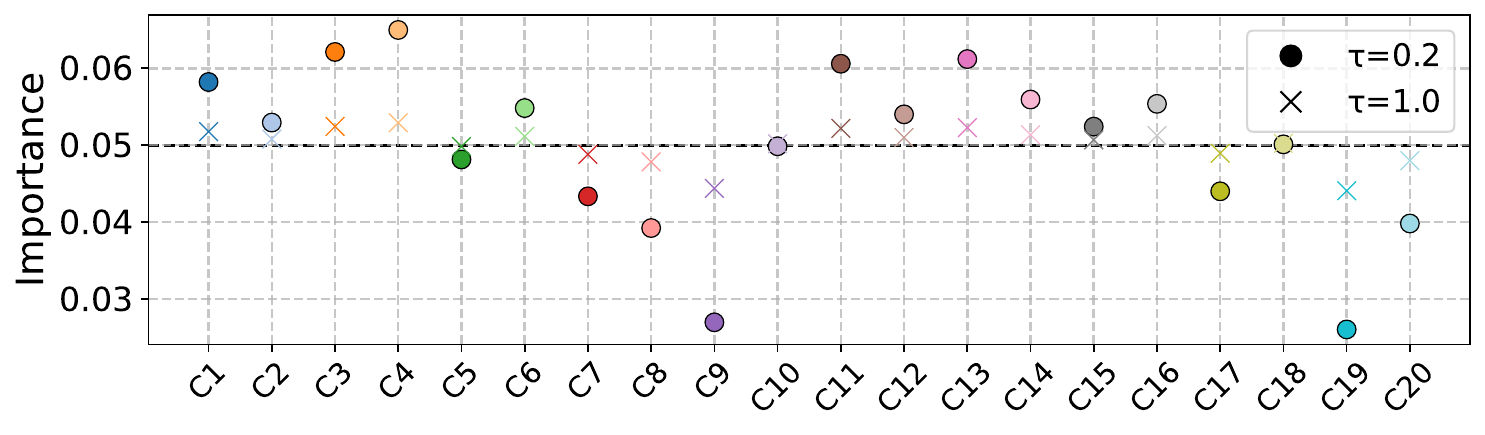}
\vspace{-6mm}
\caption{Aggregation weights under different temperature values $\tau$.}
\vspace{-3mm}
\label{importance}
\end{figure}

\textbf{Importance-aware aggregation of distributed contextual knowledge.}
To account for heterogeneous client relevance, we assign aggregation weights based on the similarity between each client layout and the query.
Fig.~\ref{importance} shows the resulting weight distributions under two temperature values.
A smaller $\tau$ produces a more concentrated distribution and gives greater influence to a few highly relevant clients.
A larger $\tau$ produces a smoother distribution and allows more clients to contribute to the final prediction.
Thus, $\tau$ controls the tradeoff between emphasizing the most relevant clients and preserving complementary information from multiple clients.
Importantly, collaboration enables the server to exploit demonstrations distributed across clients without placing them all into a single context.
This effectively increases the contextual information available for ICL while respecting the LAVM context length constraint and keeping private demonstrations locally stored.

\vspace{-1mm}
\begin{figure}[h]
\centering
\begin{minipage}{0.23\textwidth}
\centering
\includegraphics[
    width=0.95\linewidth,
    height=0.79\linewidth,
    keepaspectratio=false,
    trim=1 1 1 1,
    clip
]{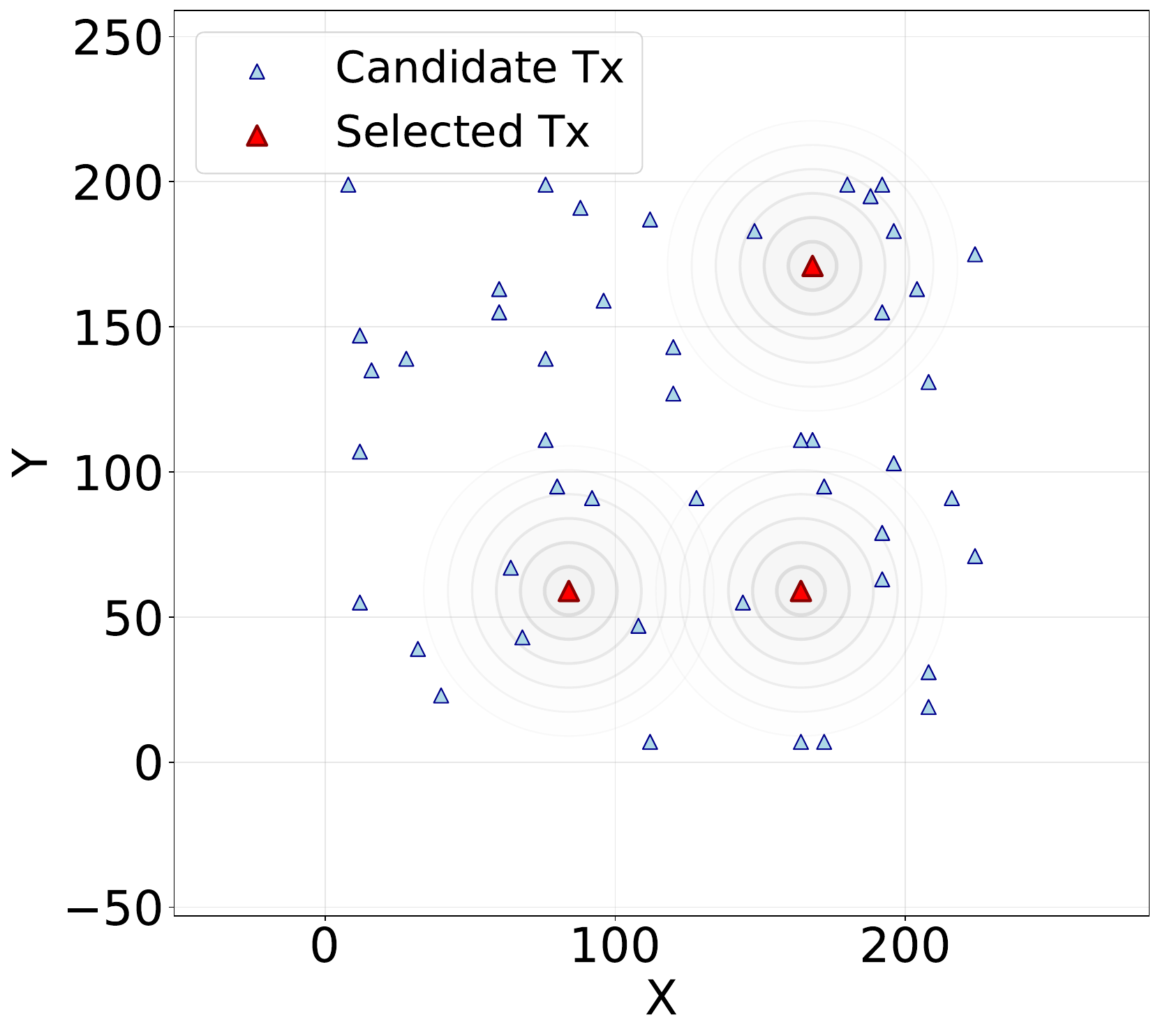}
\vspace{-2mm}
\caption{Selected Tx locations.}
\label{BTx}
\end{minipage}
\hfill
\begin{minipage}{0.23\textwidth}
\centering
\includegraphics[
    width=0.95\linewidth,
    trim=1 1 1 1,
    clip
]{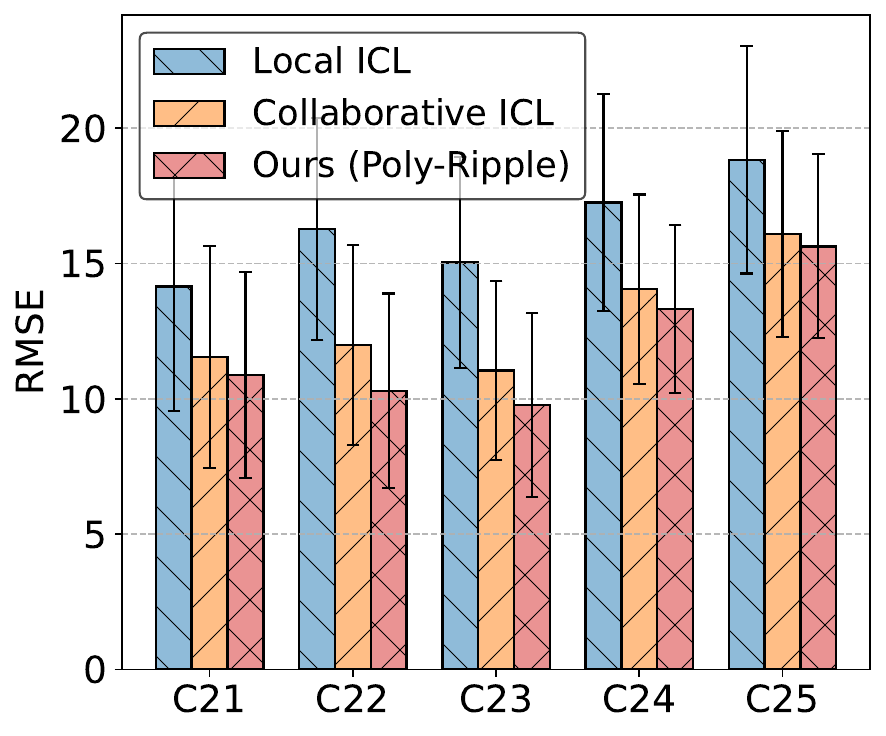}
\vspace{-2mm}
\caption{RMSE comparison.}
\label{RMSE_case2}
\end{minipage}
\end{figure}

\vspace{-2mm}
\textbf{Quantitative performance comparison.}
Fig.~\ref{RMSE_case2} compares local ICL, collaborative ICL with uniform aggregation, and \textsc{Poly-Ripple} on the five unseen layouts, C21 to C25.
\textsc{Poly-Ripple} achieves the lowest RMSE on all five layouts.
Collaborative ICL with uniform aggregation generally outperforms local ICL, showing that combining predictions derived from distributed demonstrations can reduce the error caused by each client's limited context.
This result supports the intuition that additional relevant demonstrations improve ICL performance even when the context length prevents them from being processed together.
It is also consistent with the aggregation gain characterized by our theoretical analysis.
The further improvement of \textsc{Poly-Ripple} over collaborative ICL with uniform aggregation shows that performance depends not only on the number of clients but also on their relevance to the query.
Uniform aggregation gives less relevant clients the same influence as better aligned clients.
In contrast, \textsc{Poly-Ripple} uses importance-aware aggregation to emphasize predictions supported by more relevant demonstrations while retaining complementary information from other clients, thereby achieving consistently lower estimation error.

\textbf{Qualitative performance.}
Fig.~\ref{Two_examples} presents two example queries and compares uniform aggregation with the importance-aware aggregation of \textsc{Poly-Ripple}.
The predictions from \textsc{Poly-Ripple} more closely match the target radio maps and better preserve the dominant attenuation patterns extending outward from the Tx.
This improvement is consistent with ripple-ordered autoregressive modeling, which generates radio map tokens according to their spatial propagation dependencies from the Tx.
Uniform aggregation produces larger deviations in several regions because predictions from less relevant clients receive the same weights as those from better aligned clients.
By assigning greater importance to more relevant contextual information, \textsc{Poly-Ripple} better reconstructs both the global propagation structure and local path loss variations.
\vspace{-3mm}
\begin{figure}[H]
\centering
\includegraphics[
    width=0.47\textwidth,
    trim=105 158 195 130,
    clip
]{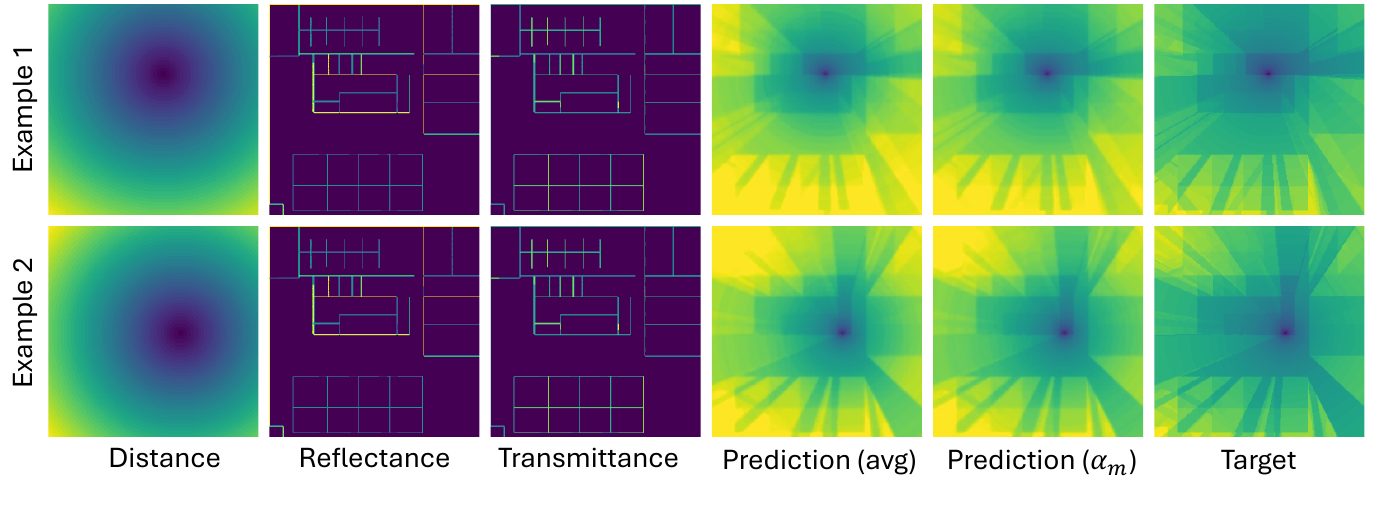}
\vspace{-4mm}
\caption{Qualitative comparison between uniform aggregation (avg) and the importance-aware ($\alpha_m$) aggregation used by \textsc{Poly-Ripple}.}
\vspace{-1mm}
\label{Two_examples}
\end{figure}

\subsection{Impact of Key Parameters}
\label{Ablation Studies}

To answer \textbf{Q2}, we study the effects of the number of participating clients and the aggregation temperature $\tau$.

\textbf{Impact of the number of clients.}
Fig.~\ref{Different_M} reports the average RMSE as the number of clients increases from $5$ to $20$.
The RMSE generally decreases as more clients participate, indicating that additional clients provide complementary contextual information for the same query.
Since the context length constraint of LAVM prevents all demonstrations from being processed within a single context, collaboration allows the server to exploit a larger set of distributed demonstrations through importance-aware aggregation.
The improvement is more pronounced when the number of clients is small, while the marginal gain gradually decreases as more clients are included.
This trend suggests that the most relevant clients provide the largest benefit, whereas additional clients may contribute redundant or less relevant information.
This trend also supports our theoretical result that collaboration improves estimation accuracy by combining complementary client predictions.

\textbf{Impact of the aggregation temperature $\tau$.}
Fig.~\ref{rmse_vs_tau} compares the RMSE under different values of $\tau$.
The results show that $\tau$ affects performance, but the variation remains relatively limited over a moderate range.
This indicates that the main gain of \textsc{Poly-Ripple} comes from importance-aware aggregation rather than careful tuning of $\tau$.
In our weighting mechanism, $\tau$ controls the concentration of the client weights.
A smaller $\tau$ emphasizes the most relevant clients, whereas a larger $\tau$ allows more clients to contribute.
Moderate values therefore provide a practical balance between client relevance and complementary information.
Thus, \textsc{Poly-Ripple} does not require extensive tuning of $\tau$ when applied to RME in a new building layout.

% \vspace{-2mm}
\begin{figure}
\centering
\begin{minipage}{0.23\textwidth}
\centering
\includegraphics[
    width=0.95\linewidth,
    height=0.79\linewidth,
    keepaspectratio=false,
    trim=1 1 1 1,
    clip
]{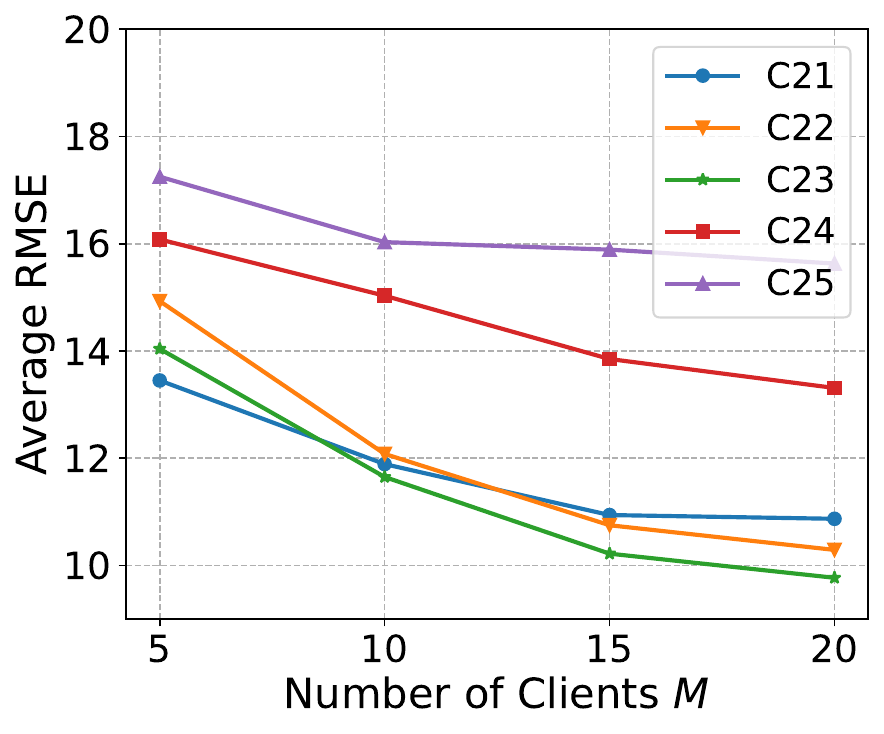}
\vspace{-2mm}
\caption{Impact of number of clients.}
\label{Different_M}
\end{minipage}
\hfill
\begin{minipage}{0.23\textwidth}
\centering
\includegraphics[
    width=0.95\linewidth,
    trim=1 1 1 1,
    clip
]{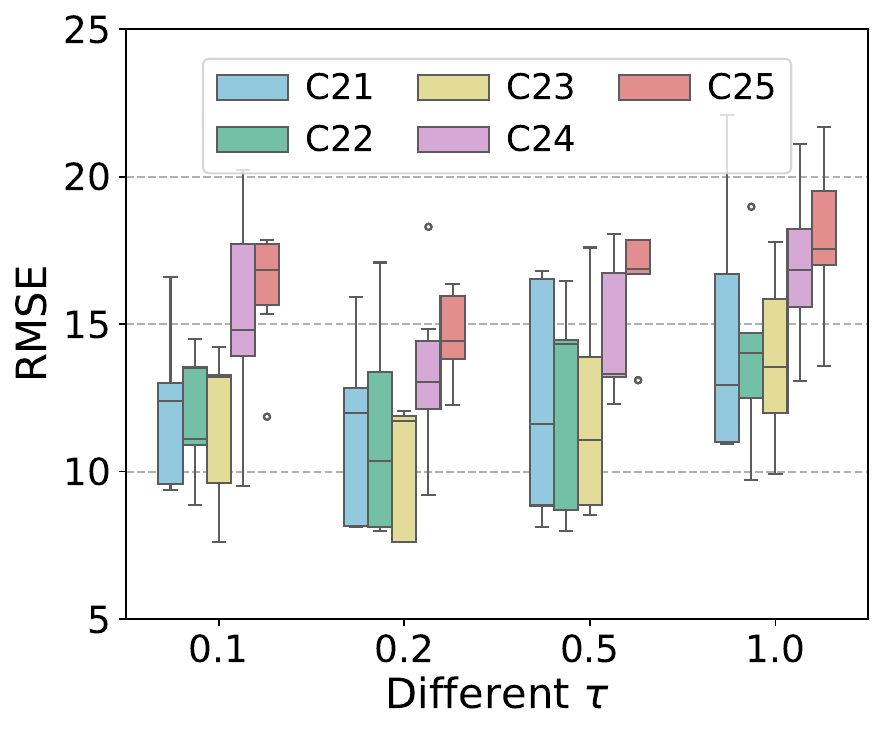}
\vspace{-2mm}
\caption{Impact of temperature $\tau$.}
\label{rmse_vs_tau}
\end{minipage}
\vspace{-3mm}
\end{figure}

\subsection{Discussion of Theory vs. Practice}
\label{subsec:theory-practice}

To answer \textbf{Q3}, we examine whether the empirical results are consistent with the theoretical analysis in Section~\ref{sec:theoretical}.

\textbf{Ripple ordering and causal structure.}
Theorem~\ref{thm:total-error-multi} shows the cumulative estimation error contains the order-dependent term
$\sum_{t=1}^{d} t\sigma_{o(t)}^{2}$.
Since coordinates predicted later receive larger coefficients, coordinates with larger uncertainty should be placed earlier in the autoregressive sequence.
Proposition~\ref{prop:optimal-order-multi} further shows that the optimal order arranges coordinates in nonincreasing order of their signal dynamic ranges.
Under Assumption~\ref{ass:monotone-sigma-multi}, this ordering corresponds to predicting spatial locations outward from the Tx, which matches the ripple order used by \textsc{Poly-Ripple}.
The qualitative results in Fig.~\ref{Two_examples} show that the predicted maps preserve the dominant outward attenuation patterns, supporting the physical motivation and theoretical advantage of ripple ordering.

\textbf{Error reduction through multi-client aggregation.}
Theorem~\ref{thm:agg-total-error} shows that uniform averaging reduces the aggregated estimation error under the cross-client uncorrelated error condition in Assumption~\ref{ass:cross-client-uncorr}.
When clients have comparable error levels, the theoretical bound yields an aggregation gain that scales with $1/M$.
This result is consistent with Fig.~\ref{RMSE_case2}, where collaborative ICL with uniform aggregation generally outperforms local ICL.
Fig.~\ref{Different_M} provides further evidence, as the average RMSE decreases when more clients participate.
The smaller marginal gain at larger $M$ is also consistent with practice because additional clients may provide increasingly redundant or less relevant contextual information.

\textbf{Client-dependent risks and importance-aware weighting.}
Theorem~\ref{thm:total-error-agg} characterizes the weighted aggregation error through
$\sum_{m=1}^{M}\alpha_m^2 r_m^2$,
where $r_m^2$ denotes a client-dependent estimation risk proxy.
Proposition~\ref{prop:alpha-optimal-after-thm3-v3} further shows that clients with lower risks should receive larger aggregation weights.
Because $r_m^2$, particularly the mismatch term $\|\delta_m\|_2^2$, is not directly observable at the server, \textsc{Poly-Ripple} uses layout similarity as a practical proxy for client-query relevance.
Fig.~\ref{all_score} illustrates the heterogeneous relevance of different client layouts, while Fig.~\ref{importance} shows the resulting nonuniform weights.
The lower RMSE of \textsc{Poly-Ripple} compared with uniform aggregation in Fig.~\ref{RMSE_case2} supports the theoretical benefit of assigning larger weights to more relevant clients.
The qualitative results in Fig.~\ref{Two_examples} further show that importance-aware aggregation produces radio maps that more closely match the targets.

\textbf{Theory-guided temperature design.}
Proposition~\ref{prop:alpha-optimal-after-thm3-v3} shows that clients with lower estimation risks should receive larger aggregation weights.
This result motivates our similarity-based weighting design, where layout similarity serves as a practical proxy for the unobservable client risk.
The temperature $\tau$ is introduced to control the concentration of these relevance-based weights.
Fig.~\ref{rmse_vs_tau} shows that the RMSE varies with $\tau$ but remains relatively stable over a moderate range.
This result indicates that $\tau$ primarily adjusts the weight concentration, while the main performance gains arise from multi-client collaboration and nonuniform weighting according to client relevance.

\section{Conclusion}
\label{sec:conclusion}

In conclusion, we formulate indoor RME as a collaborative ICL problem and propose \textsc{Poly-Ripple} for estimating radio maps without task-specific model updates.
Each client locally performs ICL using its private demonstrations, while ripple-ordered autoregressive modeling captures the causal structure of wireless propagation.
The server then applies importance-aware aggregation to combine complementary client predictions under privacy and context length constraints.
Our theoretical analysis characterizes the effects of ripple ordering, multi-client collaboration, and nonuniform aggregation on estimation error.
Experiments on a public RME benchmark show that \textsc{Poly-Ripple} consistently outperforms representative ICL baselines, and the empirical results support our theoretical analysis.

% \appendix
\section*{Proofs of Theoretical Results}

\subsection{Proof of Lemma~\ref{lemma:onestep-error-multi}}

\begin{proof}
Fix a client $m \in [M]$, an output coordinate $i \in [d]$, and a
feature subset $\mathcal{S} \subseteq [d]$. Client $m$ observes $K$
i.i.d.\ demonstrations satisfying
$y_{i,k}^{m}=x_{\mathcal{S}}^{k\top}w_{\mathcal{S}}^{\star}
+\epsilon_{i,k}^{m}$.
Let $X_{\mathcal{S}} \in \mathbb{R}^{K\times|\mathcal{S}|}$ be the
stacked input matrix, and let $Y_{i,m},\varepsilon_{i,m}\in\mathbb{R}^K$
be the stacked outputs and noise, respectively. Hence,
$Y_{i,m}=X_{\mathcal{S}}w_{\mathcal{S}}^{\star}+\varepsilon_{i,m}$.
Starting from $w_{\mathcal{S},m}^{0}$, consider one gradient step on
$L_m(w)=\frac{1}{2K}\|Y_{i,m}-X_{\mathcal{S}}w\|_2^2$. Define
$G=\frac{1}{K}X_{\mathcal{S}}^\top X_{\mathcal{S}}$ and
$\delta_{\mathcal{S},m}
=w_{\mathcal{S},m}^{0}-w_{\mathcal{S}}^{\star}$.
Then,
$e_m
\triangleq
w_{\mathcal{S},m}^{1}-w_{\mathcal{S}}^{\star}
=
(I_{|\mathcal{S}|}-\eta G)\delta_{\mathcal{S},m}
+
\frac{\eta}{K}X_{\mathcal{S}}^\top\varepsilon_{i,m}.$
Next, consider an independent test sample satisfying
$y_i^{m}=x_{\mathcal{S}}^\top w_{\mathcal{S}}^\star+\epsilon_i^{m}$.
Its prediction residual is
$r_m=\epsilon_i^{m}-x_{\mathcal{S}}^\top e_m$. By the independence
and standardized-input assumptions,
$\mathbb{E}[r_m^2]
=
\sigma_i^2+\mathbb{E}\|e_m\|_2^2.$
Write $e_m=a_m+b_m$, where
$a_m=(I_{|\mathcal{S}|}-\eta G)\delta_{\mathcal{S},m}$ and
$b_m=\frac{\eta}{K}X_{\mathcal{S}}^\top\varepsilon_{i,m}$.
Conditioning on $X_{\mathcal{S}}$ gives
$\mathbb{E}[a_m^\top b_m]=0$ and
$\mathbb{E}\|b_m\|_2^2
=
\frac{\eta^2\sigma_i^2}{K^2}
\mathbb{E}\operatorname{tr}
(X_{\mathcal{S}}X_{\mathcal{S}}^\top)
=
\frac{\eta^2\sigma_i^2|\mathcal{S}|}{K}.$
Moreover, writing $G=I_{|\mathcal{S}|}+\Delta$ with
$\mathbb{E}[\Delta]=0$ yields
$a_m=(1-\eta)\delta_{\mathcal{S},m}
-\eta\Delta\delta_{\mathcal{S},m}$. Therefore,
$\mathbb{E}\|a_m\|_2^2
\le
\left(
(1-\eta)^2+\eta^2\mathbb{E}\|\Delta\|_F^2
\right)
\|\delta_{\mathcal{S},m}\|_2^2.$
For isotropic Gaussian inputs,
$\mathbb{E}\|\Delta\|_F^2
=|\mathcal{S}|(|\mathcal{S}|+1)/K$. Combining the above results gives
$\mathbb{E}\![
(
y_i^{m}-x_{\mathcal{S}}^\top w_{\mathcal{S},m}^{1}
)^2
]
\le
(
1+\frac{\eta^2|\mathcal{S}|}{K}
)\sigma_i^2
+
(
(1-\eta)^2
+
\frac{\eta^2|\mathcal{S}|(|\mathcal{S}|+1)}{K}
)
\|\delta_{\mathcal{S},m}\|_2^2,$
which proves the result.
\end{proof}

\subsection{Proof of Theorem~\ref{thm:total-error-multi}}
\label{proof:theorem1}

\begin{proof}
Fix a client $m \in [M]$. At step $t \in [d]$, apply
Lemma~\ref{lemma:onestep-error-multi} to coordinate $i=o(t)$ and
subset $\mathcal{S}=\mathcal{S}(t)$. Since
$|\mathcal{S}(t)|=t$, Assumption~\ref{ass:delta-const-multi} gives
$\mathbb{E}\![
(
y_{o(t)}^{m}
-
\boldsymbol{x}_{\mathcal{S}(t)}^{\top}
w_{\mathcal{S}(t),m}^{1}
)^2
]
\le
(1+\frac{\eta^2t}{K})\sigma_{o(t)}^2
+
(
(1-\eta)^2+\frac{\eta^2t(t+1)}{K}
)
\|\delta_m\|_2^2.$
Next, summing over $t=1,\ldots,d$ yields
$\sum_{t=1}^{d}
\mathbb{E}\![
(
y_{o(t)}^{m}
-
\boldsymbol{x}_{\mathcal{S}(t)}^{\top}
w_{\mathcal{S}(t),m}^{1}
)^2
] 
\le
\sum_{t=1}^{d}\sigma_{o(t)}^2
+
\frac{\eta^2}{K}
\sum_{t=1}^{d}t\sigma_{o(t)}^2
+
d(1-\eta)^2\|\delta_m\|_2^2
+
\frac{\eta^2\|\delta_m\|_2^2}{K}
\sum_{t=1}^{d}t(t+1).$
Since $o$ is a permutation,
$\sum_{t=1}^{d}\sigma_{o(t)}^2=\sum_{i=1}^{d}\sigma_i^2$.
Moreover,
$\sum_{t=1}^{d}t(t+1)=d(d+1)(d+2)/3$.
Substituting these identities into the preceding inequality gives
$\sum_{t=1}^{d}
\mathbb{E}\![
(
y_{o(t)}^{m}
-
\boldsymbol{x}_{\mathcal{S}(t)}^{\top}
w_{\mathcal{S}(t),m}^{1}
)^2
] 
\le
\sum_{i=1}^{d}\sigma_i^2
+
d(1-\eta)^2\|\delta_m\|_2^2
+
\frac{\eta^2\|\delta_m\|_2^2}{K}
\frac{d(d+1)(d+2)}{3}
+
\frac{\eta^2}{K}
\sum_{t=1}^{d}t\sigma_{o(t)}^2.$
% which is exactly~\eqref{eq:thm-total-multi}.
\end{proof}

\subsection{Proof of Theorem~\ref{thm:agg-total-error}}
\label{proof:theorem2}

\begin{proof}
Since every client follows the ripple order $o(t)=t$, define the
residual of client $m$ at step $t$ as
$r_t^m \triangleq y_t^m-
\boldsymbol{x}_{\mathcal{S}(t)}^\top w_{\mathcal{S}(t)}$.
Accordingly, the residual of the aggregated prediction is
$\bar r_t=\frac{1}{M}\sum_{m=1}^M r_t^m$.
By Assumption~\ref{ass:cross-client-uncorr}, the residuals of distinct
clients are uncorrelated. Therefore, all cross terms vanish, and
$\mathbb{E}[\bar r_t^2]
=
\frac{1}{M^2}
\sum_{m=1}^M
\mathbb{E}[(r_t^m)^2].$
Next, summing over all $d$ steps gives
$\sum_{t=1}^d \mathbb{E}[\bar r_t^2]
=
\frac{1}{M^2}
\sum_{m=1}^M
\sum_{t=1}^d
\mathbb{E}[(r_t^m)^2].$
For each client $m$, Corollary~\ref{corollary:ripple-multi} yields
$\sum_{t=1}^d
\mathbb{E}[(r_t^m)^2]
\le
\sum_{i=1}^d
(1+\frac{\eta^2 i}{K})\sigma_i^2
+
d(
(1-\eta)^2
+
\frac{\eta^2(d+1)(d+2)}{3K}
)
\|\delta_m\|_2^2.$
Finally, substituting this bound into the preceding equality gives
$\sum_{t=1}^d
\mathbb{E}\![
(
\frac{1}{M}\sum_{m=1}^M y_t^m
-
\boldsymbol{x}_{\mathcal{S}(t)}^\top
w_{\mathcal{S}(t)}
)^2
]
\le
\frac{1}{M^2}
\sum_{m=1}^M
[
\sum_{i=1}^d
(1+\frac{\eta^2 i}{K})\sigma_i^2
+
d(
(1-\eta)^2
+
\frac{\eta^2(d+1)(d+2)}{3K}
)
\|\delta_m\|_2^2
],$
which is exactly~\eqref{eq:agg-total-one-eq}.
\end{proof}

\subsection{Proof of Theorem~\ref{thm:total-error-agg}}
\label{proof:theorem3}

\begin{proof}
Since all clients follow the ripple order $o(t)=t$, define the
client residual at step $t$ as
$r_t^m \triangleq y_t^m-
\boldsymbol{x}_{\mathcal{S}(t)}^\top w_{\mathcal{S}(t)}$.
Under weighted aggregation,
$\hat y_t^{\mathcal A}=\sum_{m=1}^M\alpha_m y_t^m$, and hence the
aggregated residual is
$\tilde r_t=\sum_{m=1}^M\alpha_m r_t^m$.
By Assumption~\ref{ass:cross-client-uncorr}, the residuals of any two
distinct clients are uncorrelated. Therefore, the cross terms vanish,
and
$\mathbb{E}[\tilde r_t^2]
=
\sum_{m=1}^M
\alpha_m^2\mathbb{E}[(r_t^m)^2].$
Next, summing over $t=1,\ldots,d$ and exchanging the two summations
gives
$\sum_{t=1}^d\mathbb{E}[\tilde r_t^2]
=
\sum_{m=1}^M\alpha_m^2
\sum_{t=1}^d\mathbb{E}[(r_t^m)^2].$
For each client $m$, Corollary~\ref{corollary:ripple-multi} implies
$\sum_{t=1}^d\mathbb{E}[(r_t^m)^2]
\le
\sum_{i=1}^d
(1+\frac{\eta^2 i}{K})\sigma_i^2
+
d(
(1-\eta)^2+
\frac{\eta^2(d+1)(d+2)}{3K}
)\|\delta_m\|_2^2.$
Finally, substituting this bound into the preceding equality yields
$\sum_{t=1}^d
\mathbb{E}\![
(
\hat y_t^{\mathcal A}
-
\boldsymbol{x}_{\mathcal{S}(t)}^\top
w_{\mathcal{S}(t)}
)^2
]
\le
\sum_{m=1}^M\alpha_m^2
[
\sum_{i=1}^d
(1+\frac{\eta^2 i}{K})\sigma_i^2
+
d(
(1-\eta)^2+
\frac{\eta^2(d+1)(d+2)}{3K}
)\|\delta_m\|_2^2
],$
which is exactly~\eqref{eq:agg-total-wt-closed}.
\end{proof}

\subsection{Proof of Proposition~\ref{prop:alpha-optimal-after-thm3-v3}}
\label{proof:proposition2}

\begin{proof}
% Since $r_m^2>0$, the objective in \textbf{P1} is convex, while its feasible set is convex. 
Since $r_m^2>0$, \textbf{P1} is convex with a feasible set,
i.e.,
\textbf{P1} admits a unique global minimizer.
For any feasible weights satisfying
$\sum_{m=1}^{M}\alpha_m=1$, the Cauchy Schwarz inequality gives
$(\sum_{m=1}^{M}\alpha_m^2r_m^2)
(\sum_{m=1}^{M}\frac{1}{r_m^2})
\ge
(\sum_{m=1}^{M}\alpha_m)^2
=1.$
Consequently,
$\sum_{m=1}^{M}\alpha_m^2r_m^2
\ge
(\sum_{m=1}^{M}\frac{1}{r_m^2})^{-1}.$
Moreover, equality holds if and only if
$\alpha_m r_m \propto 1/r_m$, or equivalently,
$\alpha_m\propto 1/r_m^2$. Enforcing the constraint
$\sum_{m=1}^{M}\alpha_m=1$ therefore yields
$\alpha_m^\star
=
\frac{r_m^{-2}}
{\sum_{m'=1}^{M}r_{m'}^{-2}}
=
\frac{1}{r_m^2}
(
\sum_{m'=1}^{M}\frac{1}{r_{m'}^2}
)^{-1}, m\in[M].$
Finally, defining
$\kappa\triangleq
(\sum_{m'=1}^{M}1/r_{m'}^2)^{-1}$
gives $\alpha_m^\star=\kappa/r_m^2$, which is exactly
\eqref{eq:alpha-propto-with-kappa-v3}. Since $r_m^2>0$, we have
$\alpha_m^\star>0$ for every $m$, so all constraints in \textbf{P1}
are satisfied. 
% The strict convexity of the objective then guarantees that this solution is the unique global minimizer.
\end{proof}

{
% \small
\bibliographystyle{IEEEtran}
\bibliography{ref}
}

\newpage
\setcounter{page}{1}

\end{document}